\documentclass{article}

\PassOptionsToPackage{square,numbers,sort&compress}{natbib}

\usepackage[preprint]{neurips_2019}
\usepackage{zzhang}
\usepackage{graphicx}
\usepackage[utf8]{inputenc} 
\usepackage[T1]{fontenc}    
\usepackage{hyperref}       
\usepackage{url}            
\usepackage{booktabs}       
\usepackage{amsfonts}       
\usepackage{nicefrac}       
\usepackage{microtype}      
\usepackage{tabularx}
\usepackage{makecell}
\usepackage{dsfont}
\usepackage{graphicx}
\usepackage{float}
\usepackage{wrapfig}
\usepackage{caption}
\usepackage{lipsum}
\usepackage[ruled,vlined,noresetcount]{algorithm2e}
\usepackage{adjustbox}
\usepackage{xcolor}
\usepackage{footnote}
\usepackage{epstopdf}
\makesavenoteenv{tabular}
\makesavenoteenv{table}
\graphicspath{{./figure/},{./}}
\title{
  Factor Graph Neural Network}

\author{%
  Zhen Zhang\\
  Department of Computer Science\\
  National University of Singapore\\
  \texttt{zhangz@comp.nus.edu.sg}\\
  \And
  Fan Wu\thanks{Work was done during the visiting at Department of
    Computer Science, National University of Singapore.}\\
  Department of Computer Science\\
  Nanjing University,\\
  \texttt{fan01172000@gmail.com}
  \And
  Wee Sun Lee\\
  Department of Computer Science\\
  National University of Singapore\\
  \texttt{leews@comp.nus.edu.sg}\\
}

\begin{document}

\maketitle

\begin{abstract}
Most of the successful deep neural network architectures are structured, often consisting of elements like convolutional neural networks and gated recurrent neural networks. Recently, graph neural networks have been successfully applied to graph structured data such as point cloud and molecular data. These networks often only consider pairwise dependencies, as they operate on a graph structure. We generalize the graph neural network into a factor graph neural network (FGNN) in order to capture higher order dependencies. We show that FGNN is able to represent Max-Product Belief Propagation, an approximate inference algorithm on probabilistic graphical models; hence it is able to do well when Max-Product does well. Promising results on both synthetic and real datasets demonstrate the effectiveness of the proposed model. 
\end{abstract}

\section{Introduction}
\label{sec:introduction}
Deep neural networks are powerful approximators that have been extremely successful in practice. While fully connected networks are universal approximators, successful networks in practice tend to be structured, e.g. convolutional neural networks and gated recurrent neural networks such as LSTM and GRU. Convolutional neural networks capture spatial or temporal correlation of neighbouring inputs while recurrent neural networks capture temporal information, retaining information from earlier parts of a sequence. Graph neural networks (see e.g. \citep{gilmer2017neural,xu2018powerful}) have recently been successfully used with graph structured data to capture pairwise dependencies between variables and to propagate the information to the entire graph. 

Real world data often have higher order dependencies, e.g. atoms satisfy valency constraints on the number of bonds that they can make in a molecule. In this paper, we show that the graph neural network can be extended in a natural way to capture higher order dependencies through the use of the factor graph structure. A factor graph is a bipartite graph with a set of variable nodes connected to a set of factor nodes; each factor node in the graph indicates the presence of dependencies among the variables it is connected to. We call the neural network formed from the factor graph a factor graph neural network (FGNN).

Factor graphs have been used extensively for specifying Probabilistic Graph Models (PGM) which can be used to model dependencies among random variables. 
Unfortunately, PGMs suffer from scalability issues as inference in PGMs often
require solving NP-hard problems. Once a PGM has been specified or learned, an approximate inference algorithm, e.g. the Sum-Product or Max-Product Belief Propagation, is often used to infer the values of the target variables (see e.g. \citep{koller2009probabilistic}). Unlike PGMs which usually specify the semantics of the variables being modeled as well as the approximate algorithm being used for inference, graph neural networks usually learn a set of latent variables as well as the inference procedure at the same time from data, normally in an end-to-end manner; the graph structure only provides information on the dependencies along which information propagates. For problems where domain knowledge is weak, or where approximate inference algorithms do poorly, being able to learn an inference algorithm jointly with the latent variables, specifically for the target data distribution, often produces superior results.   

We take the approach of jointly learning the algorithm and latent variables in developing the factor graph neural network. The FGNN is defined using two types of modules, the Variable-to-Factor (VF) module and the Factor-to-Variable (FV) module, as shown in Figure~\ref{fig:FGNN}. These modules are combined into a layer, and the layers can be stacked together into an algorithm. We show that the FGNN is able to exactly parameterize the Max-Product Belief Propagation algorithm, which is widely used in finding approximate \map (MAP) assignment of a PGM. Thus, for situations where belief propagation gives best solutions, the FGNN can mimic the belief propagation procedure. In other cases, doing end-to-end learning of the latent variables and the message passing transformations at each layer may result in a better algorithm. Furthermore, we also show that in some special cases, FGNN can be transformed to a particular graph neural network structure, allowing simpler implementation.

We evaluate our model on a synthetic problem with constraints on the number of elements that may be present in subsets of variables to study the strengths of the approach. We then apply the method to two real 3D point cloud problems. We achieve state-of-the-art result on one problem while being competitive on the other. The promising results show the effectiveness of the proposed algorithm.

\begin{figure}[t]
  \centering
  \includegraphics[width=0.94\textwidth]{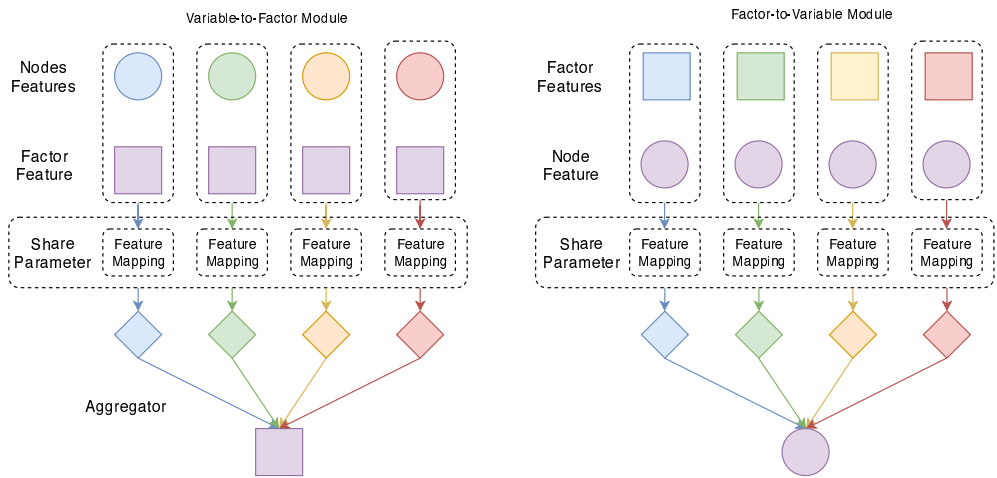}
  

  \caption{\small The structure of the Factor Graph Neural
    Network (FGNN): the Variable-to-Factor (VF) module is shown on the left while the Factor-to-Variable (FV) module is shown on the right. The VF module will let the factor to collect information from the variable nodes, and the FV module will let the nodes to receive information from their parent factors. The VF and FV modules have similar structure but different parameters. A FGNN layer usually consists of a VF layer consisting of a collection of VF modules followed by a FV layer consisting of a collection of FV modules. }  \label{fig:FGNN}
\end{figure}


\section{Background}
\label{sec:background}
\subsection{Probabilistic Graphical Model and MAP Inference}
Probabilistic Graph Models (PGMs) use graph structures to model dependencies between random
variables. These dependencies are conveniently represented using a factor graph, which is a bipartite graph $\Gcal = (\Vcal,\Ccal,\Ecal)$ where each vertex $i\in\Vcal$ in the graph is associated with a random variable $x_i$, each vertex $c\in \Ccal$ is associated with a function $f_c$ and there is an edge between variable vertex $i$ and function vertex $c$ if $f_c$ depends on variable $x_i$. \begin{wrapfigure}{r}{.4\textwidth}
	\centering
	\begin{tabular}{c c}
		\includegraphics[width=0.5\linewidth]{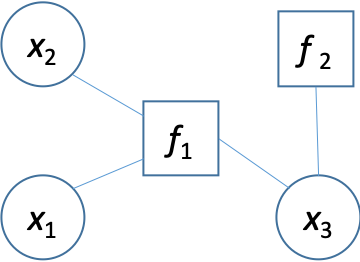}
	\end{tabular}
	\caption{ A factor graph where $f_1$ depends on $x_1$, $x_2$, and $x_3$ while $f_2$ depends on $x_3$.
	}
	\label{fig:FG}
		\vspace*{-15pt}
\end{wrapfigure}

Let $\xb$ represent the set of all variables and let $\xb_c$ represent the subset of variables that $f_c$ depends on. Denote the set of indices of variables in $\xb_c$ by $s(c)$. We consider  probabilistic models for discrete random variables of the form
\begin{align}
  \label{eq:3}
  p(\xb) = \frac{1}{Z}\exp\left[\sum_{c\in\Ccal}\theta_c(\xb_c
  ) +\sum_{i\in \Vcal}\theta_i(x_i) \right],
\end{align}
where $\exp(\theta_c(\cdot))$, $\exp(\theta_i(\cdot))$ are positive functions called potential functions (with $\theta_c(\cdot)$, $\theta_i(\cdot)$ as the corresponding log-potential functions) and $Z$ is a normalizing constant. 

The goal of MAP inference \citep{koller2009probabilistic} is to find the
assignment which maximizes $p(\xb)$, that is 
\begin{align}
  \label{eq:map_def}
  \xb^{\ast} = \argmax_{\xb}\sum_{c\in\Ccal}\theta_c(\xb_c
  ) +\sum_{i\in \Vcal}\theta_i(x_i).
\end{align}
The combinatorial optimization problem \eqref{eq:map_def} is NP-hard in general,
and thus it is often solved using approximate methods. One common method is 
Max-Product Belief Propagation, which is an iterative method formulated as
\begin{align}
  \label{eq:max_product}
  b_i(\xb_i) = \theta_i(x_i) + \sum_{c: i \in s(c)}m_{c\rightarrow
  i}(x_i); \quad 
  m_{c\rightarrow
  i}(x_i) =  \max_{\hat \xb_c: \hat x_i = x_i}\bigg[\theta_c(\hat \xb_c) +
  \sum_{i'\in s(c), i'\neq i}b_{i'}(\hat x_{i'})\bigg].
\end{align}
Max-product type algorithms are fairly effective in practice, achieving moderate accuracy in
various problems \citep{weiss2001optimality, felzenszwalb2006efficient,globerson2008fixing}. 
\subsection{Related Works} 
Various graph neural network models have been proposed for graph structured data. These include methods based  on the graph Laplacian \citep{bruna2013spectral, defferrard2016convolutional,kipf2016semi}, using gated networks \citep{li2015gated}, and using various other neural networks structures for updating the information \citep{duvenaud2015convolutional,battaglia2016interaction,kearnes2016molecular,schutt2017quantum}. In \citep{gilmer2017neural}, it was shown that these methods can be viewed as performing message passing on pairwise graphs and can be generalized to a Message Passing Neural
Network (MPNN) architecture. In this work, we seek to go beyond pairwise interactions by using message passing on factor graphs.

The PointNet \citep{qi2017pointnet} provides permutation invariant functions on a set of points instead of a graph. It propagates information from all nodes to a global feature vector, and allows new node features to be generated by appending the global feature to each node. The main benefit of
PointNet is its ability to easily capture global information. However, due to
a lack of local information exchange, it may lose the ability to represent local details.

Several works have applied graph neural networks to point cloud data, including the EdgeConv method
\citep{dgcnn} and Point Convolutional Neural Network (PointCNN)
\citep{li2018pointcnn}. We compare our work with these methods in the
experiments.  Besides applications to the molecular problems
\citep{battaglia2016interaction,duvenaud2015convolutional,gilmer2017neural,kearnes2016molecular,li2015gated,schutt2017quantum}
graph neural networks have also been applied to many other problem
domains such as combinatorial optimization \citep{khalil2017learning},
point cloud processing \citep{li2018pointcnn,dgcnn} and binary code similarity detection \citep{xu2017neural}.




\section{Factor Graph Neural Network}
\label{sec:methodology}
Previous works on graph neural networks focus on
learning pairwise information exchanges. The Message Passing Neural Network (MPNN) \citep{gilmer2017neural} provides a framework for deriving different graph neural network algorithms by modifying the message passing operations.  We aim at enabling the network to efficiently encode higher order features and to propagate information between higher order factors and the nodes by performing message passing on a factor graph. We describe the FGNN network and show that for specific settings of the network parameters we obtain the Max-Product Belief Propagation algorithm. Finally, we show that for certain special factor graph structures, FGNN can be represented as a pairwise graph neural network, allowing simpler implementation.

\subsection{Factor Graph Neural Network}
\label{sec:fact-mess-pass}
First we give a brief
introduction to the Message Passing Neural Network (MPNN), and then we
propose one MPNN architecture which can be easily extended to a
factor graph version.

Given a graph $\Gcal =
(\Vcal, \Ncal)$, where $\Vcal$ is a set of nodes and $\Ncal: \Vcal
\mapsto 2 ^{\Vcal}$ is the adjacency list, assume that each node is
associated with a feature vector $\fb_i$ and  each edge $(i,j)$ with $i\in \Vcal$ and 
$j\in\Ncal(i)$ is associated with an edge feature vector $\eb_{ij}$. 
Then a message passing neural network layer is defined in \citep{gilmer2017neural} as 
\begin{equation}
    \mb_{i} = \sum_{j \in \Ncal(i)}\Mcal(\fb_i, \fb_j,
    \eb_{ij}), \qquad \tilde \fb_{i} = \Ucal_t(\fb_i, \mb_i), \label{eq:update}
\end{equation}
where $\Mcal$ and $\Ucal$ are usually parameterized by neural networks. 
The summation in \eqref{eq:update} can be replaced with other
aggregation function such as maximization \citep{dgcnn}. The main
reason to replace summation is that summation may be corrupted by a
single outlier, while the maximization operation is more robust. Thus
in our paper we also choose to use the maximization as aggregation function.

There are also multiple choices of the architecture of $\Mcal$ and
$\Ucal$. In our paper, we propose a MPNN architecture as follows
\begin{align}
  \label{eq:MPNN_spec}
  \tilde \fb_{i} = \max_{j \in \Ncal(i)}\Qcal(\eb_{ij})\Mcal(\fb_i, \fb_j)
  ,
\end{align}
where $\Mcal$ maps feature vectors to a length-$n$ feature vector,
and $\Qcal(\eb_{ij})$ maps $\eb_{ij}$ to a $m\times n$ weight
vector. Then by matrix multiplication and aggregation a new feature of
length $m$ can be generated. 

The MPNN encodes unary and pairwise edge features, but higher order features are not directly encoded. Thus we extend the
MPNN by introducing extra factor nodes.  Given a factor graph
$\Gcal = (\Vcal,\Ccal,\Ecal)$, a group of unary features $[\fb_i]_{i \in
  \Vcal}$ and a group of factor features $[\gb_c]_{c\in\Ccal}$, assume
that for each edge $(c, i)\in \Ecal$, with $c\in\Ccal, i \in \Vcal$, there is an associated edge feature vector $[t_{ci}]$. Then, the Factor Graph Neural
Network layer on $\Gcal$ can be extended from \eqref{eq:MPNN_spec} as
shown in Figure \ref{fig:FGNN_detail}. 

\begin{figure}[H]
  \centering
    \adjustbox{valign=c}{\begin{minipage}[t]{0.50\linewidth}
    \begin{algorithm}[H]
      \SetAlgoLined
      \KwData{$\Gcal = (\Vcal,\Ccal,\Ecal)$,  $[\fb_i]_{i \in
          \Vcal}$, $[\gb_c]_{c\in\Ccal}$ and $[t_{ci}]_{(c,i) \in \Ecal}$  }
      \KwResult{$[\tilde \fb_i]_{i \in
          \Vcal}$, $[\tilde \gb_c]_{c\in\Ccal}$ }
      \textbf{Variable-to-Factor:}\\
      $\tilde \gb_c = \Agg\limits_{i:(c,i) \in \Ecal}
      \Qcal(\tb_{ci}|\Phi_{\text{VF}})\Mcal([\gb_c, f_i ]|\Theta_{\text{VF}})$\;
      \textbf{Factor-to-Variable:} \\
      $\tilde f_i = \Agg\limits_{c:(c,i) \in \Ecal}\Qcal(\tb_{ci}|\Phi_{\text{FV}})\Mcal([\ \gb_c,
      f_i]|\Theta_{\text{FV}})$\;
      \caption{\small The FGNN layer}
    \end{algorithm}
  \end{minipage}}
\hfill
  \adjustbox{valign=c}{\begin{minipage}[t]{0.47\linewidth}
    \centering
    \includegraphics[width=\textwidth]{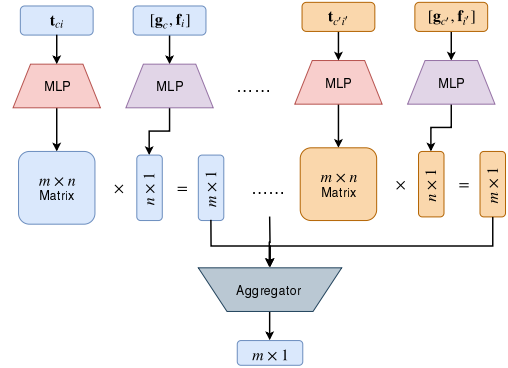}    
  \end{minipage}}

  \caption{ The Factor Graph Neural
    Network. \textbf{Left: } The pseudo code for the FGNN
    layer. We can see that the Variable-to-Factor
  module and the Factor-to-Variable modules are MPNN layers with similar
  structure but different parameters. \textbf{Right: } The detailed
  architecture for our Variable-to-Factor or Factor-to-Variable
  module.}\label{fig:FGNN_detail}
\end{figure}

\subsection{FGNN for Max-Product Belief Propagation}
\label{sec:belief-propagation}
MAP inference over the PGM is NP-hard in general, and thus it is often
approximately solved by the Max-Product Belief Propagation method. In
this section, we will prove that the Max-Product Belief Propagation can
be exactly parameterized by the FGNN. The sketch of the proof is as
follows. First we show that arbitrary higher order potentials can be
decomposed as maximization over a set of rank-1 tensors, and that the decomposition can be represented by a FGNN layer. After the decomposition, a single Max-Product Belief Propagation iteration only requires two operations: (1) maximization over rows or columns of a matrix, and (2) summation over a group of features. We show that the two operations can be exactly parameterized by the FGNN and that $k$ Max-Product iterations can be simulated using $k$ FGNN layers plus a linear layer at the end.

In the worst case, the size of a potential function grows exponentially with the number of variables that it depends on. In such cases, the size of the FGNN produced by our construction will correspondingly grow exponentially. However, if the potential functions can be well approximated as the maximum of a moderate number of rank-1 tensors, the corresponding FGNN will also be of moderate size. In practice, the potential functions may be unknown and only features of the of the factor nodes are provided; FGNN can learn the approximation from data, potentially exploiting regularities such as low rank approximations if they exist.

\paragraph{Tensor Decomposition} For discrete variables $x_1,\ldots,x_n$, a rank-1 tensor is a product of univariate functions of the variables $\prod_{i=1}^n\phi_i(x_i)$. A tensor can always be decomposed as a finite sum of rank-1 tensors \cite{kolda2009tensor}. This has been used to represent potential functions, e.g. in \citep{pmlr-v70-wrigley17a}, in conjunction with sum-product type inference algorithms. For max-product type algorithms, a decomposition as a maximum of a finite number of rank-1 tensors is more appropriate. It has been shown in  \citep{kohli2010energy} (as stated next) that there is always a finite decomposition of this type. 
    
\begin{lemma}[\citep{kohli2010energy}]
  Given an arbitrary potential function $\phi_c(\mathbf{x}_c)$, there
  exists a variable $z_c$ which takes a finite number of values and a set of univariate potentials associated with each value of $z_c$, $\{\phi_{ic}(x_i,
  z_c)|x_i\in \xb_c\}$, s.t.
  \begin{equation}
    \label{eq:lemma1}
    \log \phi_c(\mathbf{x}_c) = \log \max_{z_c}\prod_{i\in s(c)}\phi_{ic}(x_i,
    z_c)= \max_{z_c}\sum_{i\in s(c)}\log\phi_{ic}(x_i,
    z_c)
  \end{equation}
\end{lemma}

Using ideas from \citep{kohli2010energy}, we first show that a PGM with tabular potential functions that can be converted into single layer FGNN with the non-unary potential functions represented as the maximum of a finite number of rank-1 tensors.

\begin{proposition} \label{propos:decomposition}
  A factor graph $\Gcal = (\Vcal,\Ccal,\Ecal)$ with variable log potentials $\theta_i(x_i)$ and 
  factor log potentials $\log\phi_c(\xb_c)$ can be converted into a factor graph $\Gcal'$ with the same variable potentials and the corresponding decomposed factor log-potentials $\log \phi_{ic}(x_{i}, z_c)$ using a one-layer FGNN. 
\end{proposition}

The proof of Proposition~\ref{propos:decomposition} and the following
two propositions can be found in the supplementary material. With the decomposed higher order potential, one iteration of of the Max Product algorithm \eqref{eq:max_product} can be rewritten using the following two equations:
\begin{subequations}
  \begin{align}
  \label{eq:to_z}
    b_{c\rightarrow i}(z_c)& =  \sum_{i'\in s(c),i'\neq i}\max_{x_i'}\left[\log\phi_{i'c}(x_{i'}, z_c) + b_{i'}(x_{i'})\right], \\ \label{eq:to_x}
  b_i(x_i)& =  \theta_i(x_i)  + \sum_{c: i\in s(c)}\max_z[\log\phi_{ic}(x_i, z_c)+ b_{c\rightarrow i}(z_c)].
\end{align}\label{eq:decomposed}
\end{subequations}
Given the log potentials represented as a set of rank-1 tensors at each factor node, we show that each iteration of the Max-Product message passing update can be represented by a Variable-to-Factor (VF) layer and a Factor-to-Variable (FV) layer, forming a FGNN layer, followed by a linear layer (that can be absorbed into the VF layer for the next iteration).

With decomposed log-potentials, belief propagation only requires two operations: (1) maximization over rows or columns of a matrix; (2) summation over a group of features. 

We first show that the maximization operation in (\ref{eq:to_z}) (producing max-marginals) can be done using neural networks that can be implemented by the $\Mcal$ units in the VF layer.

\begin{proposition}\label{propos:matrix_max}
  For arbitrary real valued feature
  matrix $\mathbf{X} \in \mathbb{R}^{m\times n}$ with $x_{ij}$ as its
  entry in the $i^{\text{th}}$ row and $j^{\text{th}}$ column, the feature mapping operation
  $
  \hat \xb = [\max_{j}x_{ij}]_{i=1}^m
  $
  can be exactly parameterized with a 2$\log_2 n$-layer neural network with Relu as activation function and at most $2n$ hidden units. 
\end{proposition}

Following the maximization operations, equation (\ref{eq:to_z}) requires summation of a group of features. However, the VF layer uses max instead of sum operators to aggregate features produced by $\Mcal$ and $\Qcal$ operators. Assuming that the $\Mcal$ operator has performed the maximization component of equation (\ref{eq:to_z}) producing max-marginals, Proposition \ref{propos:feature_sum} shows how the $\Qcal$ layer can be used to produce a matrix $\mathbf{W}$ that converts the max-marginals into an intermediate form to be used with the max aggregators. The output of the max aggregators can then be transformed with a linear layer ($\mathbf{Q}$ in Proposition \ref{propos:feature_sum}) to complete the computation of the summation operation required in equation (\ref{eq:to_z}). Hence, equation (\ref{eq:to_z}) can be implemented using the VF layer together with a linear layer that can be absorbed in the $\Mcal$ operator of the following FV layer.


\begin{proposition}\label{propos:feature_sum}
  For arbitrary non-negative valued feature
  matrix $\mathbf{X} \in \mathbb{R}_{\geqslant 0}^{m\times n}$ with $x_{ij}$ as its
  entry in the $i^{\text{th}}$ row and $j^{\text{th}}$ column, there
  exists a constant tensor $\mathbf{W} \in
  \mathbb{R}^{m \times n \times mn}$ that can be used to transform $\mathbf{X}$ into an intermediate representation $y_{ik} = \sum_{j}x_{ij}w_{ijk}$, such that after maximization operations are done to obtain $\hat y_k = \max_{i}y_{ik}$, we can use another constant matrix $\mathbf{Q}\in \mathbb{R}^{n
    \times mn }$ to obtain 
  \begin{equation}
    [\sum_ix_{ij}]_{j=1}^{n} = \Qb[\hat y_k]_{k=1}^{mn}.
  \end{equation}
\end{proposition}


Equation (\ref{eq:to_x}) can be implemented in the same way as equation (\ref{eq:to_z}) by the FV layer. First the max operations are done by the $\Mcal$ units to obtain max-marginals. The max-marginals are then transformed into an intermediate form using the $\Qcal$ units which are further transformed by the max aggregators. An additional linear layer is then sufficient to complete the summation operation required in equation (\ref{eq:to_x}). The final linear layer can be absorbed into the next FGNN layer, or as an additional linear layer in the network in the case of the final Max-Product iteration.

We have demonstrated that we can use $k$ FGNN layers and an additional linear layer to parameterize $k$ iterations of max product belief propagation. Combined with Proposition \ref{propos:decomposition} which shows that we can use one layer of FGNN to generate tensor decomposition, we have the following corollary.

\begin{corollary}
  The max-product proposition in \eqref{eq:max_product} can be exactly
  parameterized by the FGNN. 
\end{corollary}





\paragraph{Transformation into Graph Neural Network} For graph
structures where there exists a perfect matching in the
Factor-Variable bipartite graph,
i.e. there exists an invertible function $h: \Vcal \mapsto \Ccal$, s.t.
$
  \forall i \in \Vcal,  i \in s(c), c= h(i); ~~ \forall c \in \Ccal,  s(c) \ni i=h^{-1}(c),
$
a FGNN layer can be implemented as a MPNN layer by stacking the variable
feature and factor feature as follows\footnote{More details and
  derivations are provided in the supplementary file.},
\begin{align}
  \label{eq:stack_mp_nn}
  \left[
  \begin{array}{l}
    \hat \gb_{h(i)} \\
    \hat \fb_i 
  \end{array}
  \right] = \max_{i' \in \Ncal(i)}\left[
  \begin{array}{ll}
    \Qcal(\tb_{h(i')i'}|\Phi_{\text{VF}}) & 0\\
     0 & \Qcal(\tb_{h(i')i'}|\Phi_{\text{FV}})
  \end{array}
  \right]\left[
  \begin{array}{l}
    \Mcal([\gb_{h(i')}, f_{i'} ]|\Theta_{\text{VF}}) \\
    \Mcal([\gb_{h(i')}, f_{i'} ]|\Phi_{\text{FV}})
  \end{array}
  \right],
\end{align}

\begin{wrapfigure}{r}{0.4\textwidth}
  \centering
  \includegraphics[width=0.4\textwidth]{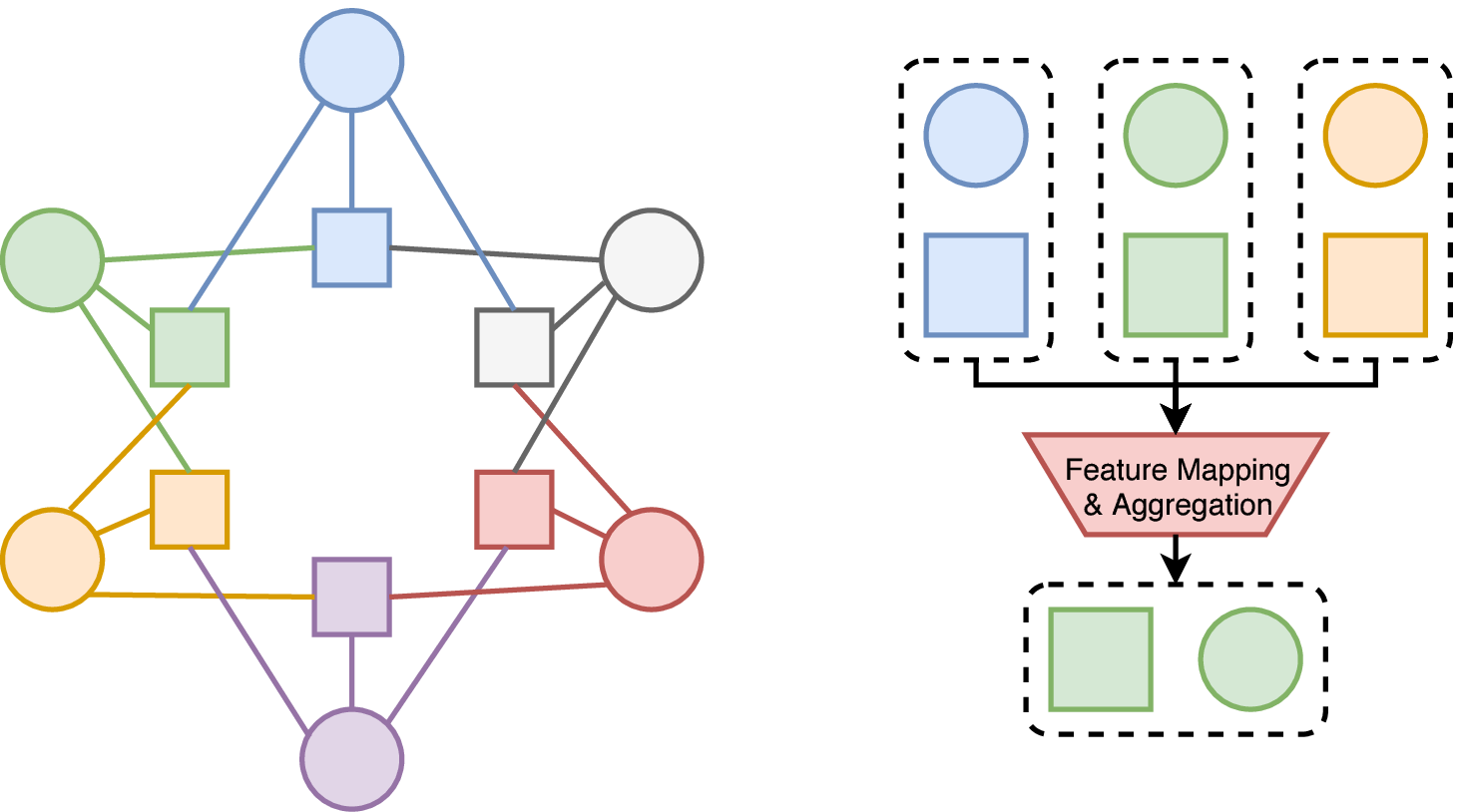}
  \caption{ Example of transformation into a MPNN.}\label{fig:factor_node_shrinking}
  		\vspace*{-15pt}
\end{wrapfigure}
where the neighborhood relation $\Ncal(i)$ is defined as
$
  \Ncal(i) = \{j| j \in s(c), c =h(i) ~\text{or } i \in s(c'), c'=h(j)\}.
$
The transformation may connect unrelated factors and
nodes together (\ie a tuple $(c, i)$ s.t. $i\not\in s(c)$). We can add additional gating functions to remove irrelevant connections, or view the additional connections as providing useful additional approximation capabilities for the network. Various graph structures satisfy the above
constraints in practice. Figure \ref{fig:factor_node_shrinking} shows an example of
the factor graph transformation. Another example is the graph
structure in point cloud segmentation proposed by \citet{dgcnn}, where for each
point, a factor is defined to include the point and its $k$-nearest
neighbor. 



\section{Experiment}
\label{sec:experiment}
In this section, we evaluate the models constructed using FGNN for
two types of tasks: MAP inference over higher order PGMs, and point cloud segmentation. 

\subsection{MAP Inference over PGMs}
\label{sec:synthetic-data}

\paragraph{Data}

We construct three synthetic datasets for this experiment in the following manner.
We start with a chain structure with length 30 where all nodes take binary states. The node potentials are all randomly generated from the uniform distribution over $[0, 1]$. We use pairwise potentials that encourage two adjacent nodes to take state $1$, i.e. the potential function gives high value to configuration $(1,1)$ and low value to all others. 
 In the first dataset, the pairwise potentials are fixed, while in the other two datasets, they are randomly generated. We then add the budget higher order potential \citep{martins2015ad} at every node; these potentials allow at most $k$ of the 8 variables that are within their scope to take the state $1$. For the first two datasets, the value $k$ is set to $5$ and in the third dataset, it is set randomly. Parameters that are not fixed are provided as input factor features; complete description of the datasets is included in the supplementary material.
Once the data item is generated, we use a branch-and-bound solver
\citep{martins2015ad} to find the exact MAP solution, and use it 
as label to train our model. 

We test the ability of the proposed model to find the MAP solutions, and compare the results with other graph neural network based methods including PointNet\citep{qi2017pointnet} and EdgeConv \citep{dgcnn} as well as several specific MAP inference solver including AD3 \citep{martins2015ad} which solves a linear programming relaxation of the problem, and Max-Product Belief Propagation \citep{weiss2001optimality}, implemented by \citep{mooij2010libdai}. Both AD3 and Max-Product are approximate inference algorithms and are run with the correct model for each instance.

\paragraph{Architecture and training details} 
We use a multi-layer
factor graph neural network with architecture
\textsc{FGNN(64) - Res[FC(64) - FGNN(64) - FC(64)] - MLP(128)
  - Res[FC(64) - FGNN(64) - FC(128)] - FC(256) - Res[FC(64) - FGNN(64) -
  FC(256)] - FC(128) - Res[FC(64) - FGNN(64) - FC(128)] - FC(64) -
  Res[FC(64) - FGNN(64) - FC(64)] - FGNN(2)}. 
Here one FGNN($C_{\text{out}}$) is a FGNN layer with $C_{\text{out}}$
as output feature dimension with ReLU \citep{nair2010rectified} as activation. One FC($C_{\text{out}}$) is a fully
connect layer with $C_{\text{out}}$ as output feature dimension and
ReLU as activation.  \text{Res[]} is a neural network with
residual link from its input to output \citep{he2016deep}.




The model is implemented using pytorch \citep{paszke2017automatic}
trained with Adam optimizer with initial learning rate
$\text{lr}=3\times 10^{-3}$ and after each epoch, 
lr
is decreased by a factor of $0.98$\footnote{Code is at \url{https://github.com/zzhang1987/Factor-Graph-Neural-Network}}. For PointNet and EdgeConv
are trained using their recommended hyper parameter (for point cloud segmentation problems). For all the models listed in Table \ref{tab:res_syn}, we train for $50$ epoches after which all models achieve convergence.

\paragraph{Results} We compare the prediction of each method with that
of the exact MAP solution. The percentage agreement (including the
mean and standard deviation) are provided in Table
\ref{tab:res_syn}. Our model achieves far better result on both
Dataset1 and Dataset2 compared with all other methods. The performance
on Dataset3 is also comparable to that of the LP-relaxation. 

There is no comparison with PointNet\citep{qi2017pointnet} and DGCNN\citep{dgcnn} on Dataset2 and Dataset3 because it is generally difficult for them to handle edge features associated with random pairwise potential and high order potential. Also, due to the convergence problems of Max-Product solver on Dataset3, this experiment was not carried out.

Max-Product performs poorly on this problem. So, in this case, even though it is possible for FGNN to emulate the Max-Product algorithm, it is better to learn a different inference algorithm.


\begin{table}[htbp]
  \small
  \centering
  \setlength{\tabcolsep}{3pt}

  \begin{tabular}{lccccc}
    \toprule
    \thead{Agreement \\ with MAP}  &\thead{PointNet\citep{qi2017pointnet}} & \thead{DGCNN\citep{dgcnn}} & LP Relaxation  & Max-Product & \thead{Ours}\\
    \midrule
    Dataset1 & 42.6$\pm$ 0.007 & 60.2$\pm$0.017 & 80.7$\pm$0.025 &53.0$\pm$0.101 & \textbf{92.5}$\pm$0.020\\
    Dataset2 & -- & -- & 83.8$\pm$0.024 &54.2$\pm$0.164 & \textbf{89.1}$\pm$0.017\\
    Dataset3 & -- & -- & \textbf{88.2}$\pm$0.011 & -- & 87.7$\pm$ 0.013\\
    \bottomrule
  \end{tabular}
  \caption{ MAP inference results on synthetic datasets of PGMs}
  \label{tab:res_syn}
\end{table} 

We have used FGNN factor nodes that depend on 8 neighboring variable nodes to take advantage of known dependencies. We also did a small ablation study on the size of the high order potential functions on Dataset1. The resulting accuracies are 81.7 and 89.9 when 4 and 6 variables are used instead. This shows that knowing the correct size of the potential function can give advantage.

\subsection{Point Cloud Segmentation}
\label{sec:3d-part-segmentation}

\paragraph{Data} We use the Stanford
Large-Scale 3D indoor Spaces Dataset (S3DIS) \citep{armeni20163d} for
semantic segmentation and the ShapeNet part dataset
\citep{yi2016scalable} for part segmentation.

The S3DIS dataset includes 3D scan point clouds for 6 indoor areas
including 272 rooms labeled with 13 semantic categories. We follow
the same setting as \citet{qi2017pointnet} and \citet{dgcnn}, where
each room is split into blocks with size 1 meter $\times$  1 meter,
and each point is associated with a 9D vector (absolute spatial coordinates,
RGB color, and normalized spatial coordinates). For each block, 4096
points are sampled during training phase, and all points are used
in the testing phase. In the experiment we use the same 6-fold cross
validation setting as \citet{qi2017pointnet,dgcnn}.

The ShapeNet part dataset contains 16,811 3D
shapes from 16 categories, annotated with 50 parts. From each shape  
2048 points are sampled and most 
data items are with less than
six parts. In the part segmentation, we follow the official
train/validation/test split 
provided by \citet{chang2015shapenet}.

\paragraph{Architecture and training details} For both semantic segmentation and part
segmentation, for each point in the point
cloud data, we define a factor which include the point and its
$k$-nearest neighbor and use the feature of that point as factor
feature, and the edge feature is computed as $t_{ci}=\gb_c[0:3] - \fb_i[0:3]$. 
In our network, we set $k=16$. A nice property of such factor construction procedure is that 
it is easy to find a perfect matching in the Factor-Variable bipartite graph, and then
the FGNN can be efficiently implemented as MPNN.

Then a multi-layer FGNN with the architecture \textsc{Input  - 
FGNN(64)  -  FC(128)  -  Res[FC(64) - FGNN(64) - FC(128)]  - 
FC(256) - Res[FC(64) - FGNN(64) - FC(256)]  -  FC(512)  - 
Res[FC(64) - FGNN(64) - FC(512)]  -  FC(512) -  GlobalPooling  - 
FC(512) - Res[FC(64) - FGNN(64) - FC(512)] - FC(256) - DR(0.5) -
FC(256) - DR(0.5) - FC(128) - FC($N$)}, where FGNN, FC and
\textsc{Res} are the same as previous section. DR($p$) is the
dropout layer with dropout rate $0.5$, and the global pooling layer (\textsc{GlobalPooling})
as the same as the global pooling in PointNet, and $N$ is the number
of classes.

Our model 
is trained with
the Adam optimizer with initial learning rate $3\times 10^{-3}$, and
after each epoch, the learning rate is decreased by a factor of
$0.98$. For semantic segmentation, we train the model with 100 epoches
and for part segmentation with batch size 8, and we train model for
200 epoches with batch size 8 for part segmentation on a single NVIDIA RTX 2080Ti card. In the
experiment, we strictly follow the same protocol as
\citet{qi2017pointnet,dgcnn} for fair comparison. More details on the
experiments are provided in the supplementary files. 

\paragraph{Results} For both 
tasks, 
we use Intersection-over-Union (IoU) on points to evaluate the performance 
of different models. We strictly follow the evaluation scheme of \citet{qi2017pointnet,dgcnn}.
The quantitative results of semantic segmentation is shown in Table \ref{tab:res_s3dis} and the quantitative results of part segmentation is shown in Table \ref{table:res_shapenet}. In semantic segmentation our algorithm attains the best performance while on part segmentation, our algorithm attains comparable performance with the other algorithms. These results demonstrate the utility of FGNN on real tasks.

\begin{table}[t]
  \small
  \centering
  \begin{tabular}{lccccccc}
    \toprule
    \thead{}  &\thead{PointNet\\ (baseline)\citep{qi2017pointnet}} & \thead{PointNet\cite{qi2017pointnet}}  & \thead{DGCNN\citep{dgcnn}} &
\thead{Point-\\CNN\citep{li2018pointcnn}}
              & \thead{Ours} \\
    \midrule
    Mean IoU & 20.1  & 47.6 & 56.1 & 57.3 & \textbf{60.0}\\
    Overall Accuracy & 53.2 & 81.1 & 84.1 & 84.5 & \textbf{85.5} \\
    \bottomrule
  \end{tabular}
  \caption{ 3D semantic segmentation results on S3DIS. The PointCNN
    model is retrained strictly following the protocol in
    \citep{qi2017pointnet,dgcnn} for fair
    comparison.}  \label{tab:res_s3dis}
\end{table} 

\begin{table}[t]
\scriptsize
\centering 
  \setlength{\tabcolsep}{2.5pt}
\begin{tabular}{lcccccccccccccccccc}
\toprule
     \thead{}&  \thead{\scriptsize Mean} & \thead{\scriptsize Areo} & \thead{\scriptsize Bag} &\thead{\scriptsize CAP} & \thead{\scriptsize Car} & \thead{\scriptsize Chair} & \thead{\scriptsize{Ear}\\ \scriptsize{Phone}} & \thead{\scriptsize Guitar} & \thead{\scriptsize Knife} & \thead{\scriptsize Lamp} & \thead{\scriptsize Laptop} & \thead{\scriptsize Motor} & \thead{\scriptsize Mug} & \thead{\scriptsize Pistol} & \thead{\scriptsize Rocket} & \thead{\scriptsize{State}\\ \scriptsize{Board}} & \thead{\scriptsize Table} \\
     \midrule
     PointNet & 83.7 & 83.4& 78.7& 82.5& 74.9& 89.6& 73.0& 91.5& 85.9& 80.8& 95.3&65.2&93.0&81.2&57.9&72.8&80.6\\
     DGCNN & 85.1 & 84.2&83.7&84.4&77.1&\textbf{90.9}&78.5&91.5&87.3&82.9&96.0&67.8&93.3&82.6&59.7&75.5&82.0\\
     PointCNN \footnote{The PointCNN model is trained with the same
  scheme except for data augmentation. During the training of
  PointCNN, coordinate data are augmented by adding a zero-mean Gaussian noise with variance $0.001^2$.} & \textbf{86.1} & 84.1 & \textbf{86.5} & 86.0 & \textbf{80.8}& 90.6 & \textbf{79.7} & \textbf{92.3} & 88.4 & \textbf{85.3} & \textbf{96.1} & \textbf{77.2} & \textbf{95.3} & \textbf{84.2} & \textbf{64.2} & \textbf{80.0} & \textbf{83.0}\\
     Ours & 84.7 & \textbf{84.7} & 84.0 & \textbf{86.1} & 78.2 & 90.8 & 70.4 & 90.8 & \textbf{88.7} & 82.4 & 95.5 & 70.6 & 94.7 & 81.0 & 56.8 & 75.3 & 80.5 \\
     \bottomrule
\end{tabular}
\caption{ Part segmentation results on ShapeNet part
  dataset.}\label{table:res_shapenet}
\vspace{-10pt}
\end{table}

\section{Conclusion}
\label{sec:conclusion}
We extend graph neural networks to factor graph neural networks, enabling the network to capture higher order dependencies among the variables. The factor graph neural networks can represent the execution of the Max-Product Belief Propagation algorithm on probabilistic graphical models, allowing it to do well when Max-Product does well; at the same time, it has the potential to learn better inference algorithms from data when Max-Product fails. Experiments on a synthetic dataset and two real datasets show that the method gives promising performance.

The ability to capture arbitrary dependencies opens up new opportunities for adding structural bias into learning and inference problems. The relationship to graphical model inference through the Max-Product algorithm provides a guide on how knowledge on dependencies can be added into the factor graph neural networks.
{\small
  \bibliographystyle{plainnat}
  \bibliography{reference/main.bib}
}

\newpage

\appendix
\begin{center}
    \Large{\textbf{Factor Graph Neural Net---Supplementary File}}
\end{center}
\section{Proof of propositions}

First we provide Lemma \ref{lemma:lossless_agg}, which will be used in
the proof of Proposition \ref{propos:decomposition} and \ref{propos:feature_sum}. 
\begin{lemma}\label{lemma:lossless_agg}
Given $n$ non-negative feature vectors $\fb_i=[f_{i0}, f_{i1}, \ldots,
f_{im}]$, where $ i=1,\ldots, n$, there
exists $n$ matrices $\Qb_i$  with shape $nm\times m$ and $n$ vector
$\hat \fb_i = \Qb_i\fb_i^{T}$, s.t.
\begin{align*}
  , \qquad [\fb_1, \fb_2, \ldots, \fb_n] =
  [\max_{i}\hat f_{i0},\max_{i}\hat f_{i1},\ldots, \max_{i}\hat f_{i,mn}].
\end{align*}
\end{lemma}
\begin{proof}
  Let
  \begin{align}
    \label{eq:6}
    \Qb_i = \left[\underbrace{\mathbf{0}^{m\times m}, \ldots, 
    \mathbf{0}^{m\times m}}_{i-1\text{ matrices}}, \mathbf{I},
    \underbrace{\mathbf{0}^{m\times m},\ldots,
    \mathbf{0}^{m\times m}}_{n - i  \text{ matrices}}\right]^{\top},
  \end{align}
  then we have that
  \begin{align*}
    \hat \fb_i = \Qb_i\fb_i^T = \left[\underbrace{0, \ldots, 0}_{(i-1)m
    \text{ zeros}},f_{i0}, f_{i1}, \ldots, f_{im}, \underbrace{0,
    \ldots, 0}_{(n - i)m \text{ zeros}}  \right]^{\top}.
  \end{align*}
By the fact that all feature vectors are non-negative, obviously we
have that
$[\fb_1, \fb_2, \ldots, \fb_n] =
  [\max_{i}\hat f_{i0},\max_{i}\hat f_{i1},\ldots, \max_{i}\hat f_{i,mn}]$.  
\end{proof}
Lemma \eqref{lemma:lossless_agg} suggests that for a group of feature
vectors, we can use the $\Qcal$ operator to produce several $\Qb$
matrices to map different vector to
different sub-spaces of a high-dimensional spaces, and then our
maximization aggregation can sufficiently gather information from
the feature groups. 

\setcounter{theorem}{1}


\begin{proposition} 
  A factor graph $\Gcal = (\Vcal,\Ccal,\Ecal)$ with variable log potentials $\theta_i(x_i)$ and 
  factor log potentials $\log\phi_c(\xb_c)$ can be converted into a factor graph $\Gcal'$ with the same variable potentials and the corresponding decomposed factor log-potentials $\log \phi_{ic}(x_{i}, z_c)$ using a one-layer FGNN. 
\end{proposition}
\begin{proof}
Without loss of generality, we assume that $\log\phi_c(\xb_c)\geqslant 1$. Then let
\begin{align}\label{eq:decomposition_to_node}
    \theta_{ic}(x_i, z_c) = \left\{ \begin{array}{ll}
    \frac{1}{|s(c)|}\log\phi_c(\xb_c^{z_c}),     & \text{if } \hat x_i = x_i^{z_c}, \\
    -c_{x_i, z_c},     &  \text{otherwise,}
    \end{array}
    \right.
\end{align}
where $c_{x_i, z_c}$ can be arbitrary real number which is larger than
$\max_{\xb_c}\theta_c(\xb_c)$. 
Obviously we will have
\begin{align}
    \log\phi_c(\xb_c) = \max_{z_c}\sum_{i\in s(c)}\theta_{ic}(x_i, z_c)
\end{align}
Assume that we have a factor $c={1,2,\ldots n}$, and each nodes can take $k$ states. Then $\xb_c$ can be sorted as 
\begin{align*} 
[&\xb_c^0=[x_1=0,x_2=0, \ldots, x_n=0],\\
 &\xb_c^1=[x_1=1,x_2=0, \ldots, x_n=0],\\
 &\ldots, \\
 &\xb_c^{k^n-1}=[x_1=k,x_2=k, \ldots, x_n=k]
],
\end{align*}
and the higher order potential can be organized as vector $\gb_c =
[\log \phi_c(\xb_c^0), \log \phi_c(\xb_c^1), \ldots, \log
\phi_c(\xb_c^{k^n - 1})]$. Then for each $i$ the item
$\theta_{ic}(x_i, z_c)$ in \eqref{eq:decomposition_to_node} have
$k^{n+1}$ entries, and each entry is either a scaled entry of the
vector $\gb_c$ or arbitrary negative number less than
$\max_{\xb_c}\theta_c(\xb_c)$.

Thus if we organize $\theta_{ic}(x_i, z_c)$ as a length-$k^{n+1}$
vector $\fb_{ic}$, then we define a $k^{n+1}\times k^n$ matrix $\Qb_{ci}$, where
if and only if the \zzth{l} entry of $\fb_{ic}$ is set to the \zzth{m} entry of
$\gb_c$ multiplied by $1/|s(c)|$, the entry of $\Qb_{ci}$ in \zzth{l} row,
\zzth{m} column will be set to $1/|s(c)|$; all the other entries of
$\Qb_{ci}$ is set to some negative number smaller than
$-\max_{\xb_c}\theta_c(\xb_c)$. Due to the assumption that
$\log\phi_c(\xb_c)\geqslant 1$, the matrix multiplication $\Qb_{ci} \gb_c$
must produce a legal $\theta_{ic}(x_i, z_c)$. 

If we directly define a $\Qcal$-network which produces the above matrices
$\Qb_{ci}$, then in the aggregating part of our network there might be
information loss. However, by Lemma \ref{lemma:lossless_agg} there
must exists a group of $\tilde \Qb_{ci}$ such that the maximization
aggregation over features $\tilde\Qb_{ci}\Qb_{ci}\gb_c$ will produce exactly
a vector representation of $\theta_{ic}(x_i, z_c), i\in s(c)$. Thus if
every $t_{ci}$ is a different one-hot vector, we can easily using one
single linear layer $\Qcal$-network to produce all
$\tilde\Qb_{ci}\Qb_{ci}$, and with a $\Mcal$-network which always
output factor feature, we are able to output a vector representation
of $\theta_{ic}(x_i, z_c), i\in s(c)$ at each factor node $c$. 
\end{proof}

Given the log potentials represented as a set of rank-1 tensors at each factor node, we need to show that each iteration of the Max Product message passing update can be represented by a Variable-to-Factor layer followed by a Factor-to-Variable layer (forming a FGNN layer). We reproduce the update equations here. 
\begin{subequations}
  \begin{align}\label{eq:bp_supp}
    b_{c\rightarrow i}(z_c) = & \sum_{i'\in s(c),i'\neq i}\max_{x_i'}\left[\log\phi_{i'c}(x_{i'}, z_c) + b_{i'}(x_{i'})\right], \\
  b_i(x_i) = & \theta_i(x_i)  + \sum_{c: i\in s(c)}\max_z[\log\phi_{ic}(x_i, z_c)+ b_{c\rightarrow i}(z_c)].
\end{align}
\end{subequations}
In the max-product updating procedure, we should keep all decomposed
$\log\phi_{i'c}(x_{i'}, z_c)$ and all unary potential
$\theta_i(x_i)$. That requires the FGNN to have the ability to
fitting identity mapping. Obviously, let the $\Qcal$ net always output
identity matrix, let the $\Mcal([\gb_c, f_i]|\Theta_{\text{VF}})$ always output
$\gb_c$, and let $\Mcal([\gb_c, f_i]|\Theta_{\text{FV}})$ always output
$f_i$, then the FGNN will be an identity mapping. As $\Qcal$ always
output a matrix and $\Mcal$ output a vector, we can use part of their
blocks as identify mapping to keep $\log\phi_{i'c}(x_{i'}, z_c)$ and
$\theta_i(x_i)$. The other blocks are used to updating
$b_{c\rightarrow i}(z_c)$ and $b_i(x_i)$. 


First we show that $\Mcal$ operators in the Variable-to-Factor layer can be used to construct the computational graph for the max-marginal operations.
\begin{proposition}
  For arbitrary real valued feature
  matrix $\mathbf{X} \in \mathbb{R}^{m\times n}$ with $x_{ij}$ as its
  entry in the $i^{\text{th}}$ row and $j^{\text{th}}$ column, the feature mapping operation
  $
  \hat \xb = [\max_{j}x_{ij}]_{i=1}^m
  $
  can be exactly parameterized with a 2$\log_2 n$-layer neural network with Relu as activation function and at most $2n$ hidden units. 
\end{proposition}

\begin{proof}
Without loss of generality we assume that $m=1$, and then we use $x_i$ to denote $x_{1i}$.
When $n=2$, it is obvious that 
\[
\max(x_1,x_2) = \textbf{Relu}(x_1-x_2) + x_2 = \textbf{Relu}(x_1-x_2) + \textbf{Relu}(x_2) - \textbf{Relu}(-x_2)
\]
and the maximization can be parameterized by a two layer neural network with 3 hidden units, which satisfied the proposition. 

Assume that when $n=2^k$, the proposition is satisfied. Then for $n=2^{k+1}$, we can find $\max(x_1, \ldots, x_{2^k})$ and $\max(x_{2^k+1}, \ldots, x_{2^{k+1}})$ using two network with $2k$ layers and at most $2^{k+1}$ hidden units. Stacking the two neural network together would results in a network with $2k$ layers and at most $2^{k+2}$. Then we can add another 2 layer network with 3 hidden units to find $\max(\max(x_1, \ldots, x_{2^k}), \max(x_{2^k+1}, \ldots, x_{2^{k+1}}))$. Thus by mathematical induction the proposition is proved.
\end{proof}

The update equations contain summations of columns of a matrix after
the max-marginal operations. However, the VF and FV layers use max
operators to aggregate features produced by $\Mcal$ and $\Qcal$
operator. Assume that the $\Mcal$ operator has produced the
max-marginals, then we  use the $\Qcal$ to produce several weight
matrix. The max-marginals are multiplied by the weight matrices to
produce new feature vectors, and the maximization aggregating function
are used to aggregating information from the new feature vectors. We
use the following propagation to show that the summations of
max-marginals can be implemented by one MPNN layer plus one
linear layer. Thus we can use the VF layer plus a linear layer to produce 
$b_{c\rightarrow i}(z_c)$ and use the FV layer plus another linear layer 
to produce $b_i(x_i)$. 
Hence to do $k$ iterations of Max Product, we need $k$ FGNN layers followed by a linear layer. 
\begin{proposition}
  For an arbitrary non-negative valued feature
  matrix $\mathbf{X} \in \mathbb{R}_{\geqslant 0}^{m\times n}$ with $x_{ij}$ as its
  entry in the $i^{\text{th}}$ row and $j^{\text{th}}$ column, there
  exists a constant tensor $\mathbf{W} \in
  \mathbb{R}^{m \times n \times mn}$ that can be used to transform $\mathbf{X}$ into an intermediate representation $y_{ik} = \sum_{ij}x_{ij}w_{ijk}$, such that after maximization operations are done to obtain $\hat y_k = \max_{i}y_{ik}$, we can use another constant matrix $\mathbf{Q}\in \mathbb{R}^{n
    \times mn }$ to obtain 
  \begin{equation}
    [\sum_ix_{ij}]_{j=1}^{n} = \Qb[\hat y_k]_{k=1}^{mn}.
  \end{equation}
\end{proposition}

\begin{proof}
The proposition is a simple corollary of Lemma
\ref{lemma:lossless_agg}. The tensor $\Wb$ serves as the same role as
the matrices $\Qb_i$ in  Lemma
\ref{lemma:lossless_agg}, which can convert the feature matrix $\Xb$
as a vector, then a simple linear operator can be used to produce the
sum of rows of $\Xb$, which completes the proof.
\end{proof}

In Lemma \ref{lemma:lossless_agg} and Proposition
\ref{propos:feature_sum}, only non-negative features are considered, while
in log-potentials, there can be negative entries. However, for the MAP
inference problem in \eqref{eq:map_def}, the transformation as follows
would make the log-potentials non-negative without changing the final
MAP assignment,
\begin{align}
  \tilde \theta_i(x_i) = \theta_i(x_i) - \min_{x_i}\theta_i(x_i),
  \qquad \tilde \theta_c(\xb_c) = \theta_c(\xb_c) - \min_{\xb_c}\theta_c(\xb_c). 
\end{align}
As a result, for arbitary PGM we can first apply the above transformation to make the log-potentials non-negative, and then our FGNN can exactly do Max-Product Belief Propagation on the transformed non-negative log-potentials.

\paragraph{Transformation to Graph Neural Network} The factor graph is a bipartite
graph $\Gcal=(\Vcal, \Ccal, \Ecal)$ among factors and nodes. If there
is a perfect matching between factors and nodes, we can use the
perfect matching between factors and nodes to transform the FGNN into a MPNN. Assume that we have a
parameterized FGNN as follows
\begin{equation}
  \begin{split}
        \tilde \gb_c & = \Agg\limits_{i:(c,i) \in \Ecal}
      \Qcal(\tb_{ci}|\Phi_{\text{VF}})\Mcal([\gb_c, f_i ]|\Theta_{\text{VF}})\\
      \tilde f_i&  = \Agg\limits_{c:(c,i) \in \Ecal}\Qcal(\tb_{ci}|\Phi_{\text{FV}})\Mcal([\ \gb_c,
      f_i]|\Theta_{\text{FV}})\label{eq:FGNN_supp}
  \end{split}
\end{equation}
When the perfect matching
exists, there must exist an invertible function $h$ which maps a
node $i\in \Vcal$ to a factor $c\in \Ccal$. Then for each $i \in
\Vcal$, we can pack the feature $[\fb_i, \gb_c], c=h(i)$ together to get a
super-node $[i, c]$.

Then we construct the edges between the super nodes. In
the FGNN \eqref{eq:FGNN_supp}, a node $i$ will exchange information
with all $c$ such that $i\in s(c)$. Thus the super-node $[i, c]$ has
to communicate with  super nodes $[i',c']$ such that $i\in
s(c')=s(h(i'))$. On the other hand, the factor $c$ will communicate with all
$i'$ such that $i'\in s(c)$, and thus the super-node $[i, c]$ has to
communicate with super nodes $[i', c']$ such that $i'\in s(c)$. Upon
these constraints, the neighbors of a super-node $[i, c], c=h(i)$ is
defined as
\begin{align}
  \label{eq:8}
  \Ncal([i, c]) = \{(i', c')| c'=h(i') \wedge (i' \in s(c) \lor i \in s(c')) \}.
\end{align}
As $h$ is a one-to-one matching function, the super node $[i, c], c=h(i)$ can
be uniquely determined by $i$, thus we can use $\Ncal(i) = \{j| j \in s(c)=
h(i) \lor  i \in s(c')=s(h(j))\}$ to represent
$\Ncal([i, c=h(i)])$.

The edge list  $\Ncal(i)$  may create link between unrelated node and
factor (\ie the node and factor do not have intersection). Thus for
each $c$ and $i$ we can
create a tag $\hat t_{ci}$ which equals 1 if $i\in s(c)$ and 0
otherwise. Without loss of generality assume $\Mcal$ and $\Qcal$ net
produces positive features,  with the tag $\hat t_{ci}$ we are able to define an
equivalent MPNN of \eqref{eq:FGNN_supp} as follows:
\begin{align*}
  \left[
  \begin{array}{l}
    \hat \gb_{h(i)} \\
    \hat \fb_i 
  \end{array}
  \right] = \max_{i' \in \Ncal(i)}\left[
  \begin{array}{ll}
    \hat t_{h(i')i'}\Qcal(\tb_{h(i')i'}|\Phi_{\text{VF}}) & 0\\
     0 & \hat t_{h(i')i'}\Qcal(\tb_{h(i')i'}|\Phi_{\text{FV}})
  \end{array}
  \right]\left[
  \begin{array}{l}
    \Mcal([\gb_{h(i')}, f_{i'} ]|\Theta_{\text{VF}}) \\
    \Mcal([\gb_{h(i')}, f_{i'} ]|\Phi_{\text{FV}})
  \end{array}
  \right].
\end{align*}
Furthermore, we can put the tag $\hat t_{ci}$ to the edge feature
$\tb_{ci}$ and let the neural network learn to reject unrelated
cluster and node, and thus the above MPNN becomes
\eqref{eq:stack_mp_nn}.
 
\section{Additional Information on MAP Inference over PGM}

We construct three datasets. All variables are binary. The instances start with a chain structure with unary potential on every node and pairwise potentials between consecutive nodes. A higher order potential is then added to every node. 

The node potentials are all randomly generated from the uniform distribution over $[0, 1]$. We use pairwise potentials that encourage two adjacent nodes to take state $1$, i.e. the potential function give high value to configuration $(1,1)$ and low value to all other configurations. The detailed setting for pairwise potential can be found in Table \ref{tab:pws} and Table \ref{tab:pws_ran}. For example, in Dataset1, the potential value for $x_1$ to take the state 0 and $x_2$ to take the state 1 is 0.2; in Dataset2 and Dataset3, the potential value for $x_1$ and $x_2$ to take the state 1 at the same time is sampled from a uniform distribution over [0, 2].

\begin{minipage}[t]{1.0\linewidth}
  \begin{minipage}{0.45\textwidth}
      \small
      \centering
      \setlength{\tabcolsep}{3pt}
      \begin{tabular}{l|cc}
        \toprule
        \thead{pairwise potential}  &\thead{$x_2=0$} & \thead{$x_2=1$} \\
        \midrule
        \thead{$x_1=0$} & 0 & 0.1  \\
        \thead{$x_1=1$} & 0.2 & 1\\
        \bottomrule
      \end{tabular}
      \captionof{table}{\small Pairwise Potential for Dataset1}
      \label{tab:pws}
  \end{minipage}
  \begin{minipage}{0.45\textwidth}
      \small
      \centering
      \setlength{\tabcolsep}{3pt}
      
      \begin{tabular}{l|cc}
        \toprule
        \thead{pairwise potential}  &\thead{$x_2=0$} & \thead{$x_2=1$} \\
        \midrule
        \thead{$x_1=0$} & 0 & 0  \\
        \thead{$x_1=1$} & 0 & U[0,2]\\
        \bottomrule
      \end{tabular}
      \captionof{table}{\small Pairwise Potential for Dataset2,3}
      \label{tab:pws_ran}
  \end{minipage}
\end{minipage}

We then add the budget higher order potential \citep{martins2015ad} at every node; these potentials allow at most $k$ of the 8 variables that are within their scope to take the state 1. For the first two datasets, the value $k$ is set to 5; for the third dataset, it is set to a random integer in \{1,2,3,4,5,6,7,8\}. 

As a result of the constructions, different datasets have different
inputs for the FGNN; for each dataset, the inputs for each instance are the parameters
of the PGM that are not fixed. For Dataset1, only the node potentials
are not fixed, hence each input instance is a factor graph with the
randomly generated node potential added as the input node feature for
each variable node. For Dataset2, randomly generate node potentials
are used as variable node features while randomly generated pairwise
potential parameters are used as the corresponding pairwise factor
node features. Finally, for Dataset3, the variable nodes, the pairwise
factor nodes and the high order factor nodes all have corresponding
input features.


\section{Extra Information On Point Cloud Segmentation}

In the point cloud segmentation experiment, there are various factors which may affects the final performance. One of the most critical part is data sampling. In both Shapenet dataset and S3DIS dataset, it is required to sample a point cloud from either a CAD model or a indoor scene block. Thus for fair comparison, all the methods are trained on the dataset sampled from original Shapenet dataset and S3Dis dataset by \citet{qi2017pointnet} \etal, and following \citet{qi2017pointnet} and \citet{dgcnn}, we do not apply any data augmentation during training. When training on Shapenet dataset, there is an official train/val/test split. We do training on the training set, and do validation after training on one epoch, and then use the model with best validation performance for evaluation on test set. For S3DIS dataset, as we are doing 6-fold cross validation, we simply run 100 epochs for each fold and do performance evaluation on the model from last epoch. The detailed comparison on the IoU of each classes are in Table \ref{tab:det_res_iou}, where for PointNet and EdgeConv, we directly use the results from \citet{qi2017pointnet,dgcnn} since we are following exactly their experiment protocal. \citet{dgcnn} did not provide the detailed IoU of each class in their paper. For PointCNN\citep{li2018pointcnn}, we rerun the experiment with exactly the same protocal as others.

\begin{table}[htbp]
  \footnotesize
  \setlength{\tabcolsep}{1.5pt}

  \centering
  \begin{tabular}[t]{lccccccccccccccc}
    \toprule
    Method
    & OA& mIoU& ceiling& floor& wall& beam& column& window& door&
                                                                  table&
                                                                         chair& sofa& bookcase& board& clutter\\
    \midrule
    PointNet &  78.5 & 47.6 & 88.0 & 88.7 & 69.3 & 42.4 & 23.1 & 47.5 &
                                                                     51.6
                                                          & 54.1 &
                                                                   42.0
                                                                       &
    9.6 & 38.2 & 29.4 & 35.2\\
    EdgeConv & 84.4 & 56.1 & -- & -- & -- & -- & -- & -- & -- & -- & -- & -- & -- & -- & --  \\
    PointCNN & 84.5 & 57.3   & 92.0 & 93.2 & 76.0 & 46.1 & 23.6 & 43.8 & 56.2 & 67.5 & 64.5 & 30.0 & 52.1 & 49.0 & 50.8 \\
    Ours & 85.5 & 60.0 & 93.0 & 95.3 & 78.3 & 59.8 & 38.3 & 55.4 &
                                                                   61.2
                                                                       &
    64.5 & 57.7& 26.7& 50.0 & 49.1 & 50.3 \\
    \bottomrule
  \end{tabular}
  \caption{\small Detailed 3D semantic results on S3DIS in overall accuracy
    (OA, \%), micro-averaged IoU (mIoU, \%) and per-class IoU (\%). The results of PointNet and EdgeConv are directly taken from their paper, and the EdgeConv \citep{dgcnn} did not provide detailed IoU of each classes.}
  \label{tab:det_res_iou}
\end{table}

\end{document}